\documentclass{article}

\usepackage{geometry}
\usepackage{microtype}
\usepackage{graphicx}
\usepackage{subfigure}
\usepackage{booktabs} 
\usepackage{url}
\usepackage{amsmath,amsthm,amssymb,amsfonts} 
\usepackage{nicefrac}
\usepackage{mathtools}
\usepackage{makecell}
\mathtoolsset{showonlyrefs}
\usepackage{enumitem}
\usepackage{wrapfig}
\usepackage{multirow}
\usepackage{hhline}
\usepackage{algorithm}
\usepackage{algorithmic}
\usepackage{balance}

\newcommand{\bta}{\boldsymbol{\eta}}

\newcommand{\bm}{\boldsymbol{m}}
\newcommand{\bg}{\boldsymbol{g}}
\newcommand{\bx}{\boldsymbol{x}}

\newcommand{\bu}{\boldsymbol{u}}
\newcommand{\bh}{\boldsymbol{h}}
\newcommand{\by}{\boldsymbol{y}}

\newcommand{\bz}{\boldsymbol{z}}

\newcommand{\bw}{\boldsymbol{w}}

\newcommand{\bgamma}{\boldsymbol{\gamma}}

\newcommand{\argmin}{\mathop{\text{argmin}}}

\newcommand{\field}[1]{\mathbb{#1}}

\newcommand{\R}{\field{R}}

\newcommand{\E}{\field{E}}

\newtheorem{lemma}{Lemma}
\newtheorem{theorem}{Theorem}
\newtheorem{cor}{Corollary}
\newtheorem{claim}{Claim}

\usepackage{natbib}
\usepackage[pdftex,bookmarksnumbered,bookmarksopen,
colorlinks,citecolor=blue,linkcolor=blue,urlcolor=blue]{hyperref}

\title{On the Last Iterate Convergence of Momentum Methods}
\author{
Xiaoyu Li\\ 
Division of System Engineering, Boston University, Boston, MA 02215, USA\\
\small{xiaoyuli@bu.edu}
\and
Mingrui Liu\\
Department of Computer Science, George Mason University, Fairfax, VA 22030, USA\\
\small{mingruil@gmu.edu}
\and
Francesco Orabona\\
Electrical \& Computer Engineering, Boston University, Boston, MA 02215, USA\\
\small{francesco@orabona.com}
}

\defcitealias{JainNN19}{Jain et al. (COLT'19)}

\begin{document}

\maketitle

\maketitle

\begin{abstract}%
SGD with Momentum (SGDM) is a widely used family of algorithms for large-scale optimization of machine learning problems. Yet, when optimizing generic convex functions, no advantage is known for any SGDM algorithm over plain SGD. Moreover, even the most recent results require changes to the SGDM algorithms, like averaging of the iterates and a projection onto a bounded domain, which are rarely used in practice.
In this paper, we focus on the convergence rate of the last iterate of SGDM. For the first time, we prove that for any constant momentum factor, there exists a Lipschitz and convex function for which the last iterate of SGDM suffers from a suboptimal convergence rate of $\Omega(\frac{\ln T}{\sqrt{T}})$ after $T$ iterations. Based on this fact, we study a class of (both adaptive and non-adaptive) Follow-The-Regularized-Leader-based SGDM algorithms with \emph{increasing momentum} and \emph{shrinking updates}. For these algorithms, we show that the last iterate has optimal convergence $O(\frac{1}{\sqrt{T}})$ for unconstrained convex stochastic optimization problems without projections onto bounded domains nor knowledge of $T$. Further, we show a variety of results for FTRL-based SGDM when used with adaptive stepsizes. Empirical results are shown as well.
\end{abstract}

\section{Introduction}

Momentum methods have become one of the most used first-order optimization algorithms in machine learning applications. When momentum is used together with Stochastic Gradient Descent (SGD), there are two main variants considered in the literature: the stochastic version of the Heavy Ball momentum (SHB)~\citep{Polyak64} and Nesterov's momentum (also called Nesterov Accelerate Gradient method)~\citep{Nesterov83}. Besides these two, there are other variations as well. For example, an exponential moving average of the (stochastic) gradients can be used to replace the gradients in the updates~\citep{KingmaB15, ReddiHSPS16, AlacaogluMMC20, LiuGY20}.

Despite this zoo of variants, due to the presence of noise, it is well-known that Stochastic Gradient Descent with Momentum (SGDM) \emph{does not} guarantee an accelerated rate of convergence of noise nor any real advantage over plain SGD on generic convex problems. For example, recent works have proved that a variant of SGD with momentum improves only the non-dominant terms in the convergence rate on some specific stochastic problems~\citep{DieuleveutFB17,JainKKNS18}. 
Moreover, often an idealized version of SGDM is used in the theoretical analysis rather than the actual SGDM people use in practice.
For example, projections onto bounded domains at each step, averaging of the iterates~\citep[e.g.,][]{AlacaogluMMC20}, and knowledge of the total number of iterations~\citep{GhadimiL12} are often assumed. The mismatch between theory and practice is concerning because, for example, it is known that in some cases the lack of projections can destroy the convergence of some algorithms~\citep{OrabonaP18}.
Overall, recent analyses seem unable to pinpoint any advantage of using a momentum term in SGD in the stochastic optimization of general convex functions.

In the following, we denote by SGDM the following updates
\begin{equation}
\label{eq:sgdm}
\bx_{t+1} = \bx_t - \eta_t \bm_t, \quad  \bm_t = \beta_t \bm_{t-1} + (1- \beta_t)\bg_t,
\end{equation}
where $0 \leq \beta_t\leq 1$.

In this paper, to show a discriminant difference between SGD and SGDM, we focus on the convergence of the last iterate. Hence, we study the convergence of the \emph{last iterate} of SGDM for unconstrained optimization of convex functions. Unfortunately, our first result is a negative one: We show that the last iterate of SGDM can have a suboptimal convergence rate for \emph{any constant momentum setting}.

Hence, motivated by the above result, we analyze yet another variant of SGDM.
We start from the very recent observation~\citep{Defazio20} that SGDM can be seen as a primal averaging procedure~\citep{NesterovS15,TaoPWT18,Cutkosky19} applied to the iterates of Online Mirror Descent (OMD)~\citep{NemirovskyY83,Warmuth97}. Based on this fact, we analyze SGDM algorithms based on the Follow-the-Regularized-Leader (FTRL) framework\footnote{FTRL is known in the offline optimization literature as Dual Averaging (DA)~\citep{Nesterov09}, but in reality, DA is a special case of FTRL when the functions are linearized.}~\citep{Shalev-Shwartz07,AbernethyHR08} and the primal averaging. The use of FTRL instead of OMD removes the necessity of projections onto bounded domains, while the primal averaging acts as a momentum term and guarantees the optimal convergence of the last iterate. The resulting algorithm has an \emph{increasing momentum} and \emph{shrinking updates} that precisely allow to avoid our lower bound.

More in detail, we prove that the expected suboptimality gap of the last iterate of FTRL-based SGDM converges at the optimal rate of $O(1/\sqrt{T})$ on convex functions, without assuming bounded domains nor the knowledge of the total number of iterations. This also disproves a more general conjecture than the one in \citep{JainNN19,JainNN21}, removing the bounded assumption. Moreover, we show that our construction is general enough to allow for an entire family of FTRL-based SGDM methods, both adaptive and non-adaptive. For example, we show that ``adaptive'' learning rates give rise to convergence rates that are adaptive to gradients, noise, and to the interpolation regime.

The rest of the paper is organized as follows: We discuss the related work in Section~\ref{sec:rel} and the setting and assumptions in Section~\ref{sec:set}. We then present our main results: the lower bound (Section~\ref{sec:lower}) and the new FTRL-based SGDM (Section~\ref{sec:ftrl_m}). Finally, in Section~\ref{sec:exp} we present an empirical evaluation of our algorithms and in Section~\ref{sec:conc} we outline a future work direction.

\section{Related Work}
\label{sec:rel}

\textbf{Stochastic Momentum Methods} SGDM has become a popular tool in deep learning and its importance has been discussed by recent studies \citep{SutskeverMDH13}. \citet{Polyak64} first proposed the use of momentum in gradient descent, calling it the Heavy-Ball method. 
In the stochastic setting, there are multiple work analyzing the use of momentum in SGD. In particular, \citet{YangLL16} prove a convergence rate of $O(1/\sqrt{T})$ for the averaged iterate in the convex setting, and for an iterate taken uniformly at random in the nonconvex setting.  \citet{LiuGY20} provide a convergence analysis for SGDM and Multistage SGDM  for smooth functions in the strongly convex and nonconvex settings. 
Also, adaptive variants of momentum methods \citep{KingmaB15, ReddiKK18, LuoXLS18} are very popular in the deep learning literature, even if their guarantees are only for the online convex optimization setting assuming a decreasing momentum factor and projections onto bounded domains. \citet{AlacaogluMMC20} recently removed the assumption of a vanishing momentum factor, but they still require projections over a bounded domain. In the non-convex and smooth case, \citet{CutkoskyO19} introduce a variant of SGDM with a variance-reduction effect and a faster convergence rate than SGD on non-convex functions, but it requires two stochastic gradients per step.

\textbf{Lower Bound}  \citet{HarveyLPR19} prove the tight convergence bound $O(\ln T/\sqrt{T})$ of the last iterate of SGD for convex and Lipschitz functions. \cite{KidambiNJK18} provide a lower bound for the Heavy Ball method for least square regression problems. To the best of our knowledge, there is no lower bound for the last iterate of SGDM in the general non-smooth non-strongly-convex setting. 

\textbf{Last Iterate Convergence of SGDM} \citet{NesterovS15} introduces a quasi-monotone subgradient method, which uses double averaging (both in Primal and Dual) based on Dual Averaging, to achieve the optimal convergence of the last iterate for the convex and Lipschitz functions. However, they just considered the batch case. This approach was then rediscovered and extended by \citet{Cutkosky19}.
Our FTRL-based SGDM is a generalization of the approach in \citet{NesterovS15} with generic regularizers and stochastic gradients.
\citet{TaoPWT18} extends \citet{NesterovS15}'s method to Mirror Descent, calling it stochastic primal averaging. They recover the same bound for convex functions, again with a bounded domain assumption.
\citet{Defazio20} points out that the sequence generated by the stochastic primal averaging \citep{TaoPWT18} can be identical to that of stochastic gradient descent with momentum for specific choices of the hyper-parameters. Accordingly, they give a Lyapunov analysis in the nonconvex and smooth case. Based on this work, \citet{JelassiD20} introduce ``Modernized dual averaging method'', which is actually equal to the one by \citet{NesterovS15}. 
They also give a similar Lyapunov analysis as in \citet{Defazio20} with specific choices of hyper-parameters 
in the non-convex and smooth optimization setting, where they assume a bounded domain and get a convergence bound $O(\ln T /\sqrt{T})$. Recently, \citet{TaoLWT21} propose the very same algorithm as in \citet{TaoPWT18} and analyze it as a modified Polyak's Heavy-ball method (already pointed out by \citet{Defazio20}). They give an analysis in the convex cases and extend it to an adaptive version, obtaining in both cases an optimal convergence of the last iterate. However, they still assume the use of projections onto bounded domains.

\begin{table}[t]
\caption{Last iterate convergence of momentum methods in convex setting}
\centering
\resizebox{\textwidth}{!}{
\begin{tabular}{c|c |c| c| c| c}
\hline\hline
Algorithm & Assumption& \makecell[c]{Bounded \\ Domain }  &  \makecell[c]{Requires \\T } &  Rate &  Reference \\ [0.5ex] 
\hline
Adaptive-HB& $ \textbf{(H3')}$ &  Yes & No & $O(\frac{1}{\sqrt{T}})$ & \cite{TaoLWT21} \\
\hline  
\multirow{2}*{SHB-IMA} & \multirow{2}*{\textbf{(H1) + (H2)}}& \multirow{2}*{No} & Yes & $O(\frac{1}{\sqrt{T}})$ &\multirow{2}*{\cite{SebbouhGD21}}  \\
\cline{4-5}
~ & ~ & ~& No & $O(\frac{\ln T}{\sqrt{T}})$ & ~\\
\hline
AC-SA & \makecell[c]{\textbf{(H2)} + \textbf{(H1) } \\ or\\ \textbf{(H2)} + Lipschitz} & No  &Yes & $O(\frac{1}{\sqrt{T}})$ & \cite{GhadimiL12} \\
\hline 
\multirow{2}*{FTRL-SGDM}& \textbf{(H3)} & No& No  & $O(\frac{1}{\sqrt{T}})$  & This paper, Corollary~\ref{cor:poly_constant_step} \\[1ex]
\cline{2-6}
~ & \textbf{(H1)+(H2)+(H3')} & No& No  & $O(\frac{\ln T}{T} + \frac{\sigma}{\sqrt{T}})$  & This paper, Corollary~\ref{cor:ada_smooth} \\[1ex]
\hline
\end{tabular}}
\label{table:comparison}
\end{table}

\textbf{Last iterate convergence rate $O(\frac{1}{\sqrt{T}})$} 
\citet{GhadimiL12} present the last iterate of AC-SA~\citep{NemirovskiJLS09, Lan12} for convex functions in the unconstrained setting, that in the Euclidean case reduces to SGD with an increasing Nesterov momentum, showing that it can achieve a convergence rate $O(\frac{1}{\sqrt{T}})$ if the number of iterations $T$ is known in advance.
\citet{SebbouhGD21} analyze Stochastic Heavy Ball-Iterave Moving Average method (SHB-IMA), which is equal to the Stochastic Heavy Ball method (SHB) with a specific choice of hyper-parameters. They prove a convergence rate for the last iterate of of $O (\frac{1}{\sqrt{T}})$ if $T$ is given in advance, and is $O(\frac{\ln T}{\sqrt{T}})$ if $T$ is unknown.
\citet{JainNN19,JainNN21} conjecture that under assumption \textbf{(H3')} (see next Section) ``for any-time algorithm (i.e., without apriori knowledge of $T$ ) expected error rate of $\frac{D G \ln T}{\sqrt{T}}$ is information-theoretically optimal'', where $D$ is the diameter of the bounded domain. This was already disproved by the results in \citet{TaoLWT21}, but here we disprove it even in the more challenging unconstrained setting.

We summarize the results on the last iterate convergence for convex optimization and their assumptions in Table~\ref{table:comparison}. The assumptions are defined in the next section.

\section{Problem Set-up}
\label{sec:set}

\paragraph{Notation}
We denote vectors by bold letters, e.g. $\bx \in \R^d$. All standard operations on the vectors, e.g., $\bx \by, \bx/\by, \sqrt{\bx}$ and $\bx < \by$, are to be considered element-wise.  We denote by $\E [ \cdot ]$ the expectation with respect to the underlying probability space and by $\E_t [ \cdot ]$ the conditional expectation with respect to the past. Any norm without particular notation in this work is the $\ell_2$ norm. 

\paragraph{Setting}
We consider the unconstrained optimization problem $\min_{ \bx \in \R^d }  \  f(\bx)$, where $f(\bx):\R^d\rightarrow \R$ is a convex function and we denote its infimum by $f^{\star}$.
We also assume to have access to a first-order black-box optimization oracle that returns a stochastic subgradient in any point $\bx \in \R^d$. In particular, we assume that we receive a vector $\bg(\bx,\xi)$ such that $\E_{\xi} \left[\bg(\bx,\xi) \right] = \nabla f(\bx)$ for any $\bx \in \R^d$. To make the notation concise, we let $\bg_t  \triangleq\bg(\bx_t, \xi_t ) $ and $\E_t [\bg_t ] = \nabla f(\bx_t) , \forall t$.  

We will make different assumptions on the objective function $f$. Sometimes, we will assume that 
\begin{itemize}
\setlength\itemsep{0em}
\item \textbf{(H1)} $f$ is $L$-smooth, that is, $f$ is continuously differentiable and its gradient is $L$-Lipschitz, i.e., $\| \nabla f(\bx) - \nabla f(\by) \| \leq L\| \bx - \by \|$. 
\end{itemize}
We also use one or more of the following assumptions on the stochastic gradients $\bg_t$. 
\begin{itemize}
\setlength\itemsep{0em}
\item \textbf{(H2)} bounded variance: $\E_t \| \bg_t - \nabla f(\bx_t) \|^2 \leq \sigma^2$. 
\item \textbf{(H3)} bounded in expectation: $\E \| \bg_t \|^2 \leq G^2$. 
\item \textbf{(H3')} $\ell_2$ bounded: $\| \bg_t \|\leq G$.
\item \textbf{(H3'')} $\ell_{\infty}$ bounded: $\| \bg_t \|_{\infty}\leq G_{\infty}$.
\end{itemize}

\section{Lower bound for SGDM}
\label{sec:lower}

First of all, as we discussed in the related work, most of the analyses of SGDM assume a vanishing momentum or a constant one. 
However, is constant momentum the best setting for stochastic optimization of convex functions, especially for the convergence of the last iterate?
For this question, it is worth remembering that the use of a constant momentum term is mainly motivated by the empirical evidence in the deep learning literature. However, deep learning objective functions are non-convex and the convex setting might be different. Also, the deep learning literature offers no theoretical explanations.

In this section, we show the surprising result that for SGD with any constant momentum, there exists a function for which the lower bound of the last iterate is $\Omega \left( \ln T/\sqrt{T}\right)$. Our proof extends the one in \citet{HarveyLPR19} to SGD with momentum.

We consider SGDM  with constant momentum factor $\beta$ in \eqref{eq:sgdm},
where $\bg_t \in \partial f(\bx_t)$ and a polynomial stepsize $\eta_t = c \cdot  t^{-\alpha}, 0 \leq \alpha \leq \frac{1}{2}$.  


Let $\mathcal{X}$ denote the Euclidean ball with radius $\frac{2c}{1- \beta}$ in $\R^T$. 
For any fixed $\beta$ and $\alpha$ and $L >0$, we introduce the following function. 
Define $f$: $\mathcal{X} \to \R $ and $\bh_i \in \R^T$ for $i \in [T+1]$ by 
\begin{equation}
\label{eq:def_of_f}
f(\bx)  = \max_{i \in [T+1]} \bh_i^T \bx,
\quad
h_{i,j} = 
\begin{cases}
a_j,  &  1 \leq j < i  \\
- b_j,  & i = j < T\\
0, & i < j \leq T
\end{cases}
\end{equation}
where $b_j = \frac{Lj^{\alpha}}{2T^{\alpha}}$ and $a_j = \frac{L(1-\beta)}{8(T-j+1)}$. 
We have that $\partial f(\bx_t)$ is the convex hull of $\bh_i: i \in \mathcal{I} (\bx)$ where $\mathcal{I}(\bx) = \{i: \bh_i ^T \bx = f(\bx)\}$.
Note that $f$ is $L$-Lipschitz over $\R^T$ since
\begin{align}
\| \bh_i \|^2 
\leq \sum_{i=1}^{T} a_i^2 + b_T^2
\leq \frac{L^2(1-\beta)^2}{64} \sum_{i=1}^{T}\frac{1}{i^2} + \frac{L^2}{4}
\leq L^2 ~. 
\end{align}
\begin{claim}
\label{clm:f_min_0}
For $f$ defined in \eqref{eq:def_of_f}, it satisfies that $\inf_{\bx \in \R^T} f(\bx) = 0.$
\end{claim}
\begin{proof}
First, since $f(0) = 0$, we have that $\inf_{\bx \in \R^T} \ f(\bx) \leq 0$. 
\\
We continue to prove this claim by contradiction. Assume that there exists $\bx^{\star} = [x_1^{\star}, x_2^{\star},\dots, x_T^{\star}]$ such that \[f(\bx^{\star}) < 0~.\] 
By the definition of $f$, it satisfies that 
\begin{equation}
\label{eq:assumption_min_f_negative}
\bh_i^T\bx^{\star} < 0,\quad \forall i \in [T+1]~.
\end{equation}
In particular, $\bh_1^T \bx^{\star} = -b_1 x^{\star}_1 < 0$. Since that $b_1$ is positive, we know that $x_1^{\star} >0$. Also, $\bh_2^T \bx^{\star} = a_1 x^{\star}_1 - b_2  x^{\star}_2 < 0$. Due to the positiveness of $a_1, x^{\star}_1$, and $b_2$, $x^{\star}_2$ has to be positive. Similarly, we have that for any $x^{\star}_j$, $j \in [T]$, $x^{\star}_j > 0$.
\\
Then, we have 
\[\bh_{T+1}^T \bx^{\star} = \sum_{j=1}^T a_j x^{\star}_j > 0~.\] However, this is contradict with \eqref{eq:assumption_min_f_negative}. 
\\
Thus, we conclude that $\inf_{\bx \in \R^T} f(\bx) = 0.$
\end{proof}
\begin{theorem}[Lower bound of SGDM]
\label{thm:lower_bound}
Fix a polynomial stepsize sequence $\eta_t =c \cdot t^{-\alpha}$, where $0 \leq \alpha \leq \frac{1}{2}$, a momentum factors $\beta \in [0,1)$, a Lipschitz constant $L > 0$ and a number of iterations $T$. Then, there exists a sequence $\bz_t$ generated by SGDM with stepsizes $\eta_t$ and momentum factor $\beta$ on the function $f$ in \eqref{eq:def_of_f}, where the $T$-th iterate satisfies
\[
f(\bz_T) - f^{\star} 
\geq 
\frac{L^2(1-\beta)^2c \ln T}{4T^{\alpha}} ~.
\]
\end{theorem}

We stress that $\ln T$ cannot be cancelled by any setting of $\beta$ or $c$. Indeed, the above lower bound can be instantiated by any $\beta$ and any $T$. Hence, for a given $\beta$, there exists $T$ large enough such that $\ln T$ is constant-times bigger than $\frac{1}{(1-\beta)^2}$.

When $\beta = 1$, the algorithm is basically staying at the initial point.  We can choose an arbitrary positive number $C>0$ and let $z_1 = C$, then
\[
f(\bz_T) - f^{\star} \geq C , \quad C>0~.
\]

We will use the following lemma in the proof. 

\begin{proof}
Define a sequence $\bz_t$ for $t \in [T+1]$ as follows: $\bz = 0$, where $s$ is a positive number decided later, and 
\begin{equation}
\label{eq:def_z_t}
\bz_{t+1}  = \bz_t - (1-\beta) \eta_t \sum_{i=1}^{t} \beta^{t-i} \bh_i~. 
\end{equation}
We will show that $\bz_t$ are exactly the updates of SGDM and $f(\bz_{T+1}) \geq  \Omega \left(\frac{\ln T}{T^{\alpha}} \right)$.
We will use the following two lemmas.
\begin{lemma}
\label{lemma:bound_z_t}
Let  $b_j = \frac{L j^{\alpha}}{2T^{\alpha}}$, $a_j = \frac{L (1 - \beta)}{8 (T-j+1)}$, and $\eta_j = c \cdot j^{-\alpha}$. $\bz_t$ is defined as in \eqref{eq:def_z_t}. Then, for $1 \leq t < j $, $z_{t,j} = 0$,  and for $t > j $, $z_{t,j} \geq \frac{L (1-\beta) c}{4T^{\alpha} }$. 
\end{lemma}
\begin{proof}
We first prove by induction that when $1 \leq t \leq j$, $z_{t,j} = 0$. 
First, $\bz_1=0$
Also, suppose it holds for $t$. 
Then, in the case of $t+1$, for any $j \geq t+1$, 
\[
z_{t+1, j} = z_{t,j} - (1-\beta)  \eta_t \sum_{i=1}^{t} \beta^{t-i} h_{i, j} = 0 - 0 = 0, 
\]
which implies $t \leq j$, $z_{t,j} = 0$ holds. 
Next, we claim that $\bz_t$ satisfies 
\begin{equation}
\begin{aligned}
\label{eq:lower_bound_z_t}
z_{t,j} 
& \geq z_{j,j} + (1-\beta) b_j \eta_j - a_j\sum_{k=j+1}^{t-1} \eta_k, \quad 1\leq j < t \leq T~. 
\end{aligned}
\end{equation}
We prove \eqref{eq:lower_bound_z_t} by induction. 
For any $t$, $z_{t,t-1}$ satisfies \eqref{eq:lower_bound_z_t} since 
\begin{align*}
z_{t+1, t} 
= z_{t,t} - (1-\beta)  \eta_t \sum_{i=1}^{t} \beta^{t-i} h_{i,t}
= - (1-\beta)  \eta_t h_{t,t} =  (1-\beta)\eta_t b_t~. 
\end{align*}
Then, suppose \eqref{eq:lower_bound_z_t} holds for any $j < t$. We show that it holds for any $j < t+1$.  We already proved for $j = t$. For $j < t$,
\begin{align}
z_{t+1, j} 
&= z_{t,j} - (1-\beta)  \eta_t \sum_{i=1}^{t} \beta^{t-i} h_{i,j}
= z_{t,j} - (1-\beta)  \eta_t \sum_{i=j}^{t} \beta^{t-i} h_{i,j} \nonumber\\
& = z_{t,j} + (1-\beta)  \eta_t \beta^{t-j} b_j - (1-\beta)  \eta_t \sum_{i=j+1}^{t} \beta^{t-i} h_{i,j}\nonumber\\
& \geq z_{t,j}  -   (1-\beta) \eta_t \sum_{i=j+1}^{t} \beta^{t-i} a_j\nonumber \\
&  \geq (1-\beta) b_j \eta_j - a_j\sum_{k=j+1}^{t-1} \eta_k -  (1-\beta)a_j \eta_t\sum_{i=j+1}^{t} \beta^{t-i}  \nonumber\\
&  \geq  (1-\beta) b_j \eta_j - a_j\sum_{k=j+1}^{t} \eta_k,  \label{eq:z_t}
\end{align}
where in the second inequality we used the induction hypothesis.

Using that $b_j = \frac{L j^{\alpha}}{2T^{\alpha}}$, $a_j = \frac{L (1 - \beta)}{8 (T-j+1)}$ 
and $\eta_j = \frac{c}{j^{\alpha}}$, we have
\begin{align}
\eqref{eq:z_t}
=  \frac{L(1-\beta) c}{2T^{\alpha}} - \frac{L(1-\beta) c}{8(T-j+1)} \sum_{k=j+1}^t \frac{1}{k^{\alpha}}~.
\label{eq:lower_alpha}
\end{align}
By Lemma~\ref{lemma:poly_bound} in the Appendix, we have that for $0 < \alpha \leq \frac{1}{2}$, 
\[
\eqref{eq:lower_alpha}
\geq \frac{L(1-\beta) c}{2T^{\alpha}} -  \frac{L(1-\beta) c}{4T^{\alpha}}
\geq \frac{L(1-\beta) c}{4T^{\alpha}}, 
\]
and for $\alpha= 0$, 
\[
\eqref{eq:lower_alpha} \geq \frac{L (1-\beta) c}{2} -  \frac{L (1-\beta) (t- j -1)c}{8(T-j+1)}\geq \frac{L (1-\beta) c}{4}~.
\] 
Thus, we have $z_{t,j} \geq \frac{L (1-\beta) c}{4T^{\alpha} } \geq \frac{L (1-\beta) c}{4T^{\alpha} }$. 
\end{proof}
\begin{lemma}
\label{lemma:remove_max}
$f(\bz_t) = \bh_t^T \bz_t$ for any $t \in [T+1]$. The subgradient oracle for $f$ at $\bz_t$ returns $\bh_t$.
\end{lemma}  
\begin{proof}
We claim that $\bh_t^T \bz_t = \bh_i^T\bz_t$ for all $i > t \geq 1$ and $\bh_t^T \bz_t > \bh_i^T \bz_t$ for all $1\leq i < t$.  

When $i>t\geq 2$, $\bz_t$ is supported on the first $t-1$ coordinates, while $\bh_t$ and $\bh_i $ agree on the first $t-1$ coordinates. 

In the case of $1 \leq i<t$, by the definition of $\bz_t$ and $\bh_t$, we have 
\begin{equation*}
\label{eq:z_t_times_h_t_minus_h_i}
\bz_t^T (\bh_t - \bh_i) 
= \sum_{j=1}^{t-1} z_{t,j} (h_{t,j} - h_{i,j}) 
= \sum_{j=i}^{t-1} z_{t,j} (h_{t,j} - h_{i,j}) 
= z_{t,i} (a_i + b_i) + \sum_{j=i+1}^{t-1} z_{t,j} a_j > 0, 
\end{equation*}
where in the last inequality we used the fact that $a_i$, $b_i$  and $z_{t,i}$ are at least non-negative.

Thus, we have proved $f(\bz_t) = \bh_t^T \bz_t$ by the definition. Moreover, $\mathcal{I} (\bz_t) = \{ i: \bh_i^T \bz_t  = f(\bz_t)\} = \{t, \dots, T+1\}$. So the subgradient evaluated at $\bz_t$ is $\bh_t$. 
\end{proof}

Now, we first get a lower bound and an upper bound of $\bz_t$ using Lemma~\ref{lemma:bound_z_t}. 
Then, by Lemma~\ref{lemma:remove_max}, we have shown that $\bz_t$ are exactly the updates of SGDM.

Thus, for $\beta \in [0,1)$, we have 
\begin{align*}
f(\bz_{T+1}) 
& = \bh_{T+1}^T \bz_{T+1} 
= \sum_{j=1}^{T} h_{T+1,j} z_{T+1,j} \\
& \geq \frac{L^2(1-\beta)^2c}{4T^{\alpha}}\sum_{j=1}^{T} \frac{1}{T-j+1} 
\geq \frac{L^2(1-\beta)^2c \ln T}{4T^{\alpha}}~.
\end{align*}
\end{proof}

\section{FTRL-based  SGDM}
\label{sec:ftrl_m}

The lower bound for the last iterate in the previous section motivates us to study a different variant of SGDM.
In particular, we aim to find a way to remove the $\ln T$ term from the convergence rate.

\citet{Defazio20} points out that the stochastic primal averaging method~\citep{TaoPWT18} (which is also an instance of Algorithm~1 in  \citet{Cutkosky19} with OMD):
\begin{align*}
\bz_{t+1} =\bz_t - \gamma_t \bg_t, 
\quad 
\bx_{t+1} = s_t \bx_t +(1- s_t ) \bz_t
\end{align*}
could be one-to-one mapped to the momentum method
\begin{align*}
\bm_{t+1} =  \beta_t \bm_t + \bg_t, 
\quad 
\bx_{t+1}= \bx_t - \alpha_t \bm_t 
\end{align*}
by setting $\gamma_{t+1} = \frac{\gamma_t - \alpha_t}{\beta_{t+1}}$.  
While this is true, the convergence rate depends on the convergence rate of OMD with time-varying stepsizes, that in turn requires to assume that $\| \bx_t  - \bx^{\star} \|^2 \leq D^2 $. This is possible only by using a projection onto a bounded domain in each step.

\begin{algorithm}[t]
\caption{FTRL-based SGDM}
\label{alg:sgdm_0}
\begin{algorithmic}[1]
\STATE \textbf{Input:} A sequence $\alpha_1, ..., \alpha_T$,  with $\alpha_1 > 0$. Non-increasing sequence $\bgamma_1, \dots,  \bgamma_{T-1}$. $\bm_0 = 0$. $\bx_1 \in \R^d$.
\FOR{$t = 1, \dots, T$}
\STATE  Get $\bg_t$ at $\bx_t$ such that $\E_t \left[\bg_t \right]= \nabla f(\bx_t )$
\STATE $\beta_t = \frac{\sum_{i=1}^{t-1} \alpha_i }{\sum_{i=1}^{t} \alpha_i } $ (Define $\sum_{i=1}^{0} \alpha_i = 0$)
\STATE $\bm_t = \beta_t \bm_{t-1} +(1- \beta_t) \bg_t$
\STATE $\bta_t = \frac{\alpha_{t+1}  \sum_{i=1}^{t} \alpha_i }{\sum_{i=1}^{t+1} \alpha_i} \bgamma_t  $
\STATE $\bx_{t+1} = \frac{\sum_{i=1}^{t} \alpha_i }{\sum_{i=1}^{t+1} \alpha_i} \bx_t + \frac{\alpha_{t+1}}{\sum_{i=1}^{t+1} \alpha_{i}} \bx_1 - \bta_t \bm_t $
\ENDFOR
\end{algorithmic}
\end{algorithm}

Thus, to go beyound bounded domains, we propose to study a new variant of SGDM which has the following form (details in Algorithm~\ref{alg:sgdm_0}), 
\begin{align*}
\bm_{t+1} =  \beta_t \bm_t + (1- \beta_t)\bg_t, 
\quad
\bx_{t+1} = s_t \bx_t - \alpha_t \bm_t ~.
\end{align*}
Note the presence of a shrinking factor $s_t\leq 1$ in the iterates in each step.
This variant comes naturally when using the primal averaging scheme with FTRL rather than OMD. Hence, we just denote it by FTRL-based SGDM.
Now, this momentum variant inherits all the good properties of FTRL. In particular, we no longer need the bounded domain assumption. Moreover, we will show that it guarantees the optimal convergence $O(\frac{1}{\sqrt{T}})$~\citep{AgarwalBRW12} of the last iterate for convex and Lipschitz functions.

\subsection{Convergence Rates for FTRL-based SGDM}

We first present a very general theorem for FTRL-based SGDM.
\begin{theorem}
\label{thm:sgdm_0}
Under the assumption in Section~\ref{sec:set}, Algorithm~\ref{alg:sgdm_0} guarantees
\begin{align*}
& \E\left[f(\bx_T ) \right]  - f^{\star}
\leq\frac{1}{\sum_{t=1}^{T} \alpha_t} \E \left[\left\| \frac{\bx_1 - \bx^{\star} }{ \sqrt{\bgamma_{T-1}}}  \right\|^2+ \sum_{t=1}^{T}\langle \bgamma_{t-1}, \alpha_t^2 \bg_t^2 \rangle\right]~. 
\end{align*}
\end{theorem}

The above theorem is very general and it gives rise to a number of different variations of the FTRL-based SGDM. In particular, we can instantiate it with the following choices.

First, we consider the most used polynomial stepsize $\frac{c}{\sqrt{t}}$ for convex and Lipschitz function, and the constant stepsize $\frac{c}{\sqrt{T}}$ if $T$ is given in advance. 
\begin{cor}
\label{cor:poly_constant_step}
Assume \textbf{(H3)} and set $\alpha_t = 1$ for all $t$. Algorithm~\ref{alg:sgdm_0} with either $\bgamma_{t-1}= \frac{c}{G\sqrt{t}} \cdot \mathbf{1}$ or $\bgamma_{t-1} = \frac{c}{G\sqrt{T}}\cdot \mathbf{1}$ guarantees 
\[
\E\left[f(\bx_T) \right] - f^{\star} \leq \frac{\|\bx_1 - \bx^{\star} \|^2 G }{c\sqrt{T}} + \frac{2cG}{\sqrt{T}}~. 
\]
\end{cor}

The above corollary tells that both of these two stepsizes give the optimal bound $O(\frac{1}{\sqrt{T}})$ for the last iterate. Next, we will show that if we use an adaptive\footnote{Even if widely used in the literature, it is a misnomer to call these stepsize ``adaptive'': an algorithm can be adaptive to some unknown quantities (if proved so), not the stepsizes.} stepsize,  Algorithm~\ref{alg:sgdm_0} gives a data-dependent convergence rate for the last iterate. We first consider a global version of the AdaGrad stepsize as in \citet{StreeterM10,LiO19, WardWB19}. 
\begin{cor}
\label{cor:adaptive_norm}
Assume \textbf{(H3')} and take  $\bgamma_t = \frac{\alpha \cdot \mathbf{1}}{\sqrt{\epsilon + \sum_{i=1}^{t} \alpha_i^2 \| \bg_i \|^2}}, \epsilon > 0, 1 \leq t \leq T$ and $\alpha_t = 1$. Then, Algorithm~\ref{alg:sgdm_0} guarantees
\begin{align*}
& \E\left[f(\bx_T) \right]  - f^{\star}
\leq \frac{1}{T} \left[\left(\frac{\|\bx_1 - \bx^{\star} \|^2}{\alpha} + 2\alpha\right) \E\sqrt{ \sum_{t=1}^{T} \| \bg_t \|^2  + \epsilon} 
+ \frac{\alpha G^2}{\sqrt{\epsilon}}\right]~.
\end{align*}
\end{cor}
We also state a result for the coordinate-wise AdaGrad stepsizes~\citep{McMahanS10,DuchiHS10}.
\begin{cor}
\label{cor:adaptive_coord}
Assume \textbf{(H3'')}  and set $\bgamma_t = \frac{\alpha }{\sqrt{\epsilon + \sum_{i=1}^{t} \alpha_i^2  \bg_i^2}}, \epsilon > 0, 1 \leq t \leq T$ and $\alpha_t = 1$. Then, Algorithm~\ref{alg:sgdm_0} guarantees
\begin{align*}
\E\left[f(\bx_T) \right] - f^{\star}
& \leq \frac{1}{T} \left[\left(\frac{\|\bx_1 - \bx^{\star} \|^2}{\alpha} + 2\alpha\right)  \sum_{j=1}^{d}\E \sqrt{\sum_{t=1}^{T} \bg_{t,j}^2 +\epsilon}
+ \frac{\alpha dG_{\infty}^2}{\sqrt{\epsilon}}\right]~.
\end{align*}
\end{cor}
The above two corollaries show that the convergence bound are adaptive to the stochastic gradients. In words, in the worst case (i.e., $\sum_{t=1}^{T} \| \bg_t \|^2 = O(T)$ and $\sum_{t=1}^{T-1} \bg_{t,j}^2 = O(T)$), the convergence rate is $O(\frac{1}{\sqrt{T}})$.  However, when the stochastic gradients are small or sparse, the rate could be much faster than $O(\frac{1}{\sqrt{T}})$. 
Moreover, the above results give very simple ways to obtain optimal convergence for the last iterate of first-order stochastic methods, that was still unclear if it could be obtained as discussed in \citet{JainNN19,JainNN21}.

Also, we now show that if in addition $f$ is smooth, the last iterate of FTRL-based momentum with the global adaptive stepsize of Corollary~\ref{cor:adaptive_norm} gives adaptive rates of convergence that interpolate between $O(\frac{1}{\sqrt{T}})$ and $O(\frac{\ln T}{T})$.
\begin{cor}
\label{cor:ada_smooth}
Assume \textbf{(H1)}. Then, under the same assumption and parameter setting of Corollary~\ref{cor:adaptive_norm},  Algorithm~\ref{alg:sgdm_0} guarantees
\begin{align*}
\E &\left[f(\bx_T)\right]  - f^{\star}
\leq \frac{C}{T} \left(\sqrt{\epsilon + 4L^2C^2 \ln^2 T + 4LC \sqrt{\epsilon}\ln T +\frac{2\alpha G^2}{\sqrt{\epsilon}}  } + \frac{\alpha G^2}{\sqrt{\epsilon}} \right)+ \frac{\sqrt{2}C \sigma}{\sqrt{T}}~.
\end{align*}
where $C  \triangleq \left( \frac{\| \bx_1 - \bx^{\star} \|^2 }{\alpha} + 2\alpha\right)$. 
\end{cor}
Observe that when $\sigma = 0$, namely when there is no noise on the gradients, the rate of $O(\frac{\ln T }{T})$ is obtained. 
As far as we know, the above theorems are the first convergence guarantees for the last iterate of momentum algorithms with adaptive learning rates in unconstrained convex optimization.

\subsection{Convergence Rate in Interpolation Regime}
\label{sec:interpolation}

Now we assume that $F(\bx)=\E_{\xi} [f(\bx,\xi)]$ and that the stochastic gradient is calculated drawing one function in each time step and calculating its gradient: $\bg_t = \nabla f(\bx_t,\xi_t)$. In this scenario, it makes sense to consider the \emph{interpolation} condition~\citep{NeedellSW15,MaBB18}
\begin{equation}
\label{eq:interpolation}
\bx^\star \in \argmin_{\bx} \ F(\bx) \Rightarrow \bx^\star \in \argmin_{\bx} \ f(\bx,\xi), \ \forall \xi~.
\end{equation}
This condition says that the problem is ``easy'', in the sense that all the functions in the expectation share the same minimizer. This case morally corresponds to the case in which there is no noise on the stochastic gradients. However, this condition seems weaker because it says that only in the optimum the gradient is exact and noisy everywhere else.
We will also assume that each function $f(\bx,\xi)$ is $L$-smooth in the first argument.
\begin{theorem}
\label{thm:interpolation}
Assume \textbf{(H1) (H3')}. Then, under the interpolation assumption in~\eqref{eq:interpolation}, Algorithm~\ref{alg:sgdm_0} with $\bgamma_t  = \frac{\alpha \cdot \mathbf{1}}{\sqrt{\epsilon + \sum_{i=1}^{t} \alpha_i^2 \| \bg_i \|^2}}, \epsilon>0$ guarantees 
\begin{align*}
\E\left[ F(\bx_T)\right] - F(\bx^{\star}) 
\leq \frac{C}{T} \left(\sqrt{\epsilon + 4L^2C^2 \ln^2 T + 4LC \sqrt{\epsilon}\ln T +\frac{2\alpha G^2}{\sqrt{\epsilon}} } + \frac{\alpha G^2}{\sqrt{\epsilon}} \right)~.
\end{align*}
where $C  \triangleq \left( \frac{\| \bx_1 - \bx^{\star} \|^2 }{\alpha} + 2\alpha\right)$. 
\end{theorem}
To the best of our knowledge, this is the first convergence rate for the last iterate of momentum methods in the interpolation setting. 

\subsection{Proofs}

Before presenting the proofs of our convergence rates, we revisit the Online-to-Batch algorithm (Algorithm~\ref{alg:online_to_batch}) by \citet{Cutkosky19}, which introduce a modification to any online learning algorithm to obtain a guarantee on the last iterate in the stochastic convex setting. 
\begin{lemma}{\citep[Theorem~1]{Cutkosky19}}
\label{lemma:ashok}
Assume \textbf{(H2)}.
Then, for all $\bx^{\star} \in D$, Algorithm~\ref{alg:online_to_batch} guarantees
\begin{equation}
\label{eq:conversion}
\E [f(\bx_T) ]  - f^{\star}
\leq \E \left[\frac{R_T (\bx^{\star})}{\sum_{t=1}^{T} \alpha_t}\right]~. 
\end{equation}
\end{lemma}

\begin{algorithm}[t]
\caption{Anytime Online-to-Batch~\citep{Cutkosky19}}
\label{alg:online_to_batch}
\begin{algorithmic}[1]
\STATE \textbf{Input:} Online learning algorithm $\mathcal{A}$ with convex domain D, $\alpha_1, ..., \alpha_T$, with $\alpha_1 > 0$.
\STATE Get Initial point $\bw_1$ from $\mathcal{A}$
\FOR{$t = 1, \dots, T$}
\STATE $\bx_t = \frac{\sum_{i=1}^{t} \alpha_i \bw_i }{\sum_{i=1}^{t} \alpha_i }$
\STATE Play $\bx_t$, receive subgradient $\bg_t$
\STATE Send $\ell_t (\bx) = \langle \alpha_t \bg_t, \bx \rangle$ to $\mathcal{A}$ as the $t$th loss
\STATE Get $\bw_{t+1}$ from $\mathcal{A}$
\ENDFOR
\end{algorithmic}
\end{algorithm}

\begin{algorithm}[t]
\caption{Anytime Online-to-Batch with FTRL}
\label{alg:online_to_batch_ftrl}
\begin{algorithmic}[1]
\STATE \textbf{Input:} $\alpha_1, ..., \alpha_T$, with $\alpha_t > 0$. $0 < \gamma_{t+1} \leq \gamma_t$. 
\STATE Initialize  $\bw_1$
\FOR{$t = 1, \dots, T$}
\STATE $\bx_t = \frac{\sum_{i=1}^{t} \alpha_i \bw_i }{\sum_{i=1}^{t} \alpha_i }$
\STATE Play $\bx_t$, receive subgradient $\bg_t$
\STATE $\bw_{t+1} = \bw_1 - \bgamma_t \sum_{i=1}^{t} \alpha_i \bg_i$
\ENDFOR
\end{algorithmic}
\end{algorithm}

Set $\psi_t({\bx})=\| \frac{\bx_1 - \bx}{\sqrt{\bgamma_{t-1}}} \|^2, 1\leq t\leq T$ as the regularizers of FTRL, where $\bgamma_{t+1} \leq \bgamma_t$ and $\bgamma_0 > 0$. Then, we write FTRL with loss $\ell_t(\bw) = \langle \alpha_t \bg_t, \bw \rangle$ as
\[
\bw_t
\in \argmin_{\bw \in \R^d} \ \psi_t (\bw) + \sum_{i=1}^{t-1} \langle \alpha_i \bg_i , \bw \rangle
= \bw_1 - \bgamma_{t-1}\sum_{i=1}^{t-1} \alpha_i \bg_i~. 
\]

We then plug FTRL into Algorithm~\ref{alg:online_to_batch} and it gives Algorithm~\ref{alg:online_to_batch_ftrl}.
Hence, using the well-known regret upper bound of FTRL (Lemma~\ref{lemma:ftrl} in the Appendix~\ref{sec:lemma_ftrl}), we get the following Lemma.
\begin{lemma}
\label{lemma: online_to_batch_ftrl}
Under the same setting with Lemma~\ref{lemma:ashok}, Algorithm~\ref{alg:online_to_batch_ftrl} guarantees
\begin{align*}
\E&\left[f(\bx_T) \right]  - f^{\star}
\leq \frac{1}{\sum_{t=1}^{T} \alpha_t} \E \left[\left\| \frac{\bu - \bx_1 }{ \sqrt{\bgamma_{T-1}}}  \right\|^2+ \sum_{t=1}^{T}\langle \bgamma_{t-1}, \alpha_t^2 \bg_t^2 \rangle\right]~. 
\end{align*}
\end{lemma}

Now we prove the connection between the FTRL-based SGDM and Algorithm~\ref{alg:online_to_batch_ftrl}. 
\begin{proof}[Proof of Theorem~\ref{thm:sgdm_0}]
We prove that the updates of $\bx_t$ in Algorithm~\ref{alg:sgdm_0} can be one-to-one mapped to the updates of $\bx_t$ Algorithm~\ref{alg:online_to_batch_ftrl} when $\bw_1 = \bx_1$. 

The update of $\bx_t $ in Algorithm~\ref{alg:online_to_batch_ftrl} can be written as following: 
\begin{align*}
\bx_{t+1}
= \frac{\sum_{i=1}^{t}\alpha_i}{\sum_{i=1}^{t+1}\alpha_i} \bx_t + \frac{\alpha_{t+1}}{\sum_{i=1}^{t+1} \alpha_i} \bw_{t+1}
= \frac{\sum_{i=1}^{t}\alpha_i}{\sum_{i=1}^{t+1}\alpha_i} \bx_t + \frac{\alpha_{t+1}}{\sum_{i=1}^{t+1} \alpha_i} \left(\bw_1 -\bgamma_t \sum_{i=1}^{t} \alpha_i \bg_i \right)~.
\end{align*}
It is enough to prove that for any $t$, $\bta_t \bm_t = \frac{\alpha_{t+1}}{\sum_{i=1}^{t+1} \alpha_i} \left(\bgamma_t \sum_{i=1}^{t} \alpha_i \bg_i \right)$.
We claim it is true and prove it by induction. 

When $t=1$, it holds that $\bta_1 \bm_1= \frac{\alpha_2 \alpha_1}{\alpha_1 + \alpha_2} \bgamma_1 \bg_1$. 
Suppose it holds for $t= k-1, k \geq 2$. Then in the case of $t = k $, we have 
\begin{align*}
& \bta_k \bm_k \\
& =\left(\frac{\sum_{i=1}^{k-1} \alpha_i }{\sum_{i=1}^{k} \alpha_i} \bm_{k-1} + \frac{\alpha_k }{\sum_{i=1}^{k} \alpha_i } \bg_k \right) \cdot  \frac{\alpha_{k+1}  \sum_{i=1}^{k} \alpha_i }{\sum_{i=1}^{k+1} \alpha_i} \bgamma_k  \\
& = \left(\frac{\sum_{i=1}^{k-1} \alpha_i }{\sum_{i=1}^{k} \alpha_i} \left( \frac{1}{\eta_{k-1}}\frac{\alpha_{k}}{\sum_{i=1}^{k} \alpha_i} \bgamma_{k-1} \sum_{i=1}^{k-1} \alpha_i \bg_i \right) + \frac{\alpha_k }{\sum_{i=1}^{k} \alpha_i } \bg_k \right)\cdot \frac{\alpha_{k+1}  \sum_{i=1}^{k} \alpha_i }{\sum_{i=1}^{k+1} \alpha_i} \bgamma_k \\
& = \frac{\alpha_{k+1}  \sum_{i=1}^{k} \alpha_i }{\sum_{i=1}^{k+1} \alpha_i} \bgamma_k  \cdot 
\left(\frac{\sum_{i=1}^{k-1} \alpha_i }{\sum_{i=1}^{k} \alpha_i} \left(  \frac{ \sum_{i=1}^{k-1} \alpha_i \bg_i}{\sum_{i=1}^{k-1} \alpha_i }\right) + \frac{\alpha_k }{\sum_{i=1}^{k} \alpha_i } \bg_k \right)
 = \frac{\alpha_{k+1} }{\sum_{i=1}^{k+1} \alpha_i} \bgamma_k \sum_{i=1}^{k} \alpha_i  \bg_i~. 
\end{align*}
where in the first equation we used the definitions of $\eta_k$ and $\bm_k$ and in the second equality we used the induction step.  So we proved the above claim. Thus, we can directly use Lemma~\ref{lemma: online_to_batch_ftrl}. 
\end{proof}

The proof of Corollary 1 is immediate and we omit it, while the proofs of Corollaries 2-4 are standard and they are presented in the Appendix~\ref{sec:proofs_cor}.
Instead, here we show the proof of Theorem~\ref{thm:interpolation}.
\begin{proof}[Proof of Theorem~\ref{thm:interpolation}]
By Theorem~\ref{thm:sgdm_0}, we have 
\begin{align}
\E\left[F(\bx_T) \right] - F(\bx^{\star})
\leq \frac{2}{T}\left(\frac{\|\bx_1 - \bx^{\star} \|^2}{\alpha} + 2\alpha\right) \sqrt{ \E \sum_{t=1}^{T} \| \nabla f(\bx_t , \xi_t )\|^2 + \epsilon} + \frac{\alpha G^2}{\sqrt{\epsilon}}\label{eq:interpolation_bound}~. 
\end{align}
Under the interpolation condition and $L$-smoothness of the functions $f$, it satisfies that 
\begin{align*}
\E \sum_{t=1}^{T} \|\nabla f(\bx_t,\xi_t)\|^2 
\leq  2L \E \left[ \sum_{t=1}^{T} \left( f(\bx_t , \xi_t) - f(\bx^{\star}, \xi_t ) \right)\right]
\leq  2L  \sum_{t=1}^{T}\E  \left[ F(\bx_t )  \right]- F(\bx^{\star} )~.
\end{align*}
Use \eqref{eq:interpolation_bound} on each $t$ to get  
\begin{align*}
\sum_{t=1}^{T}\E  \left[ F(\bx_t )  \right]- F(\bx^{\star}) 
& \leq  \sum_{t=1}^{T}\frac{1}{t} \left[\left(\frac{\|\bx_1 - \bx^{\star} \|^2}{\alpha} + 2\alpha\right) \sqrt{ \E \sum_{i=1}^{t} \| \nabla f(\bx_i , \xi_i )\|^2 + \epsilon  } + \frac{\alpha G^2}{\sqrt{\epsilon}}\right]\\
& \leq \left(\frac{\|\bx_1 - \bx^{\star} \|^2}{\alpha} + 2\alpha\right)  \cdot \left(\sqrt{ \E \sum_{t=1}^{T} \| \nabla f(\bx_t, \xi_t )\|^2 + \epsilon} + \frac{\alpha G^2}{\sqrt{\epsilon}} \right)\ln T~.
\end{align*}
Then, we solve for $\E \sum_{t=1}^{T} \|\nabla f(\bx_t,\xi_t)\|^2 $ and get 
\begin{align*}
& \E \sum_{t=1}^{T} \|\nabla f(\bx_t,\xi_t)\|^2 \\
&  \leq 4L^2\left(\frac{\|\bx_1 - \bx^{\star} \|^2}{\alpha} + 2\alpha\right)^2 \ln^2 T 
 + 4L \sqrt{\epsilon}\left(\frac{\|\bx_1 - \bx^{\star} \|^2}{\alpha} + 2\alpha\right) \ln T + \frac{2\alpha G^2}{\sqrt{\epsilon}} ~.
\end{align*}
Using this expression in \eqref{eq:interpolation_bound}, we have the stated bound. 
\end{proof}

\section{Empirical Results}
\label{sec:exp}

\begin{figure*}[!t]
\centering
\includegraphics[width=0.32\textwidth]{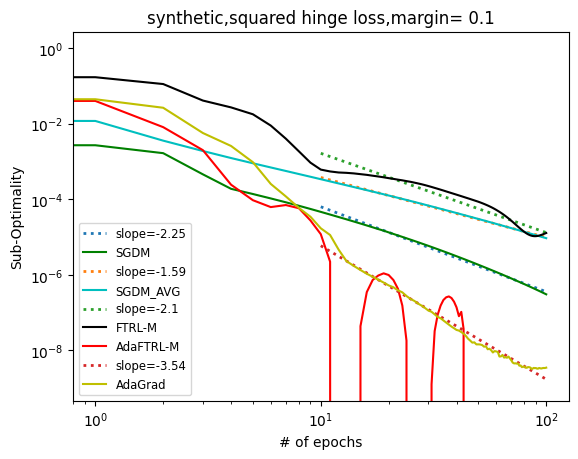}
\includegraphics[width=0.32\textwidth]{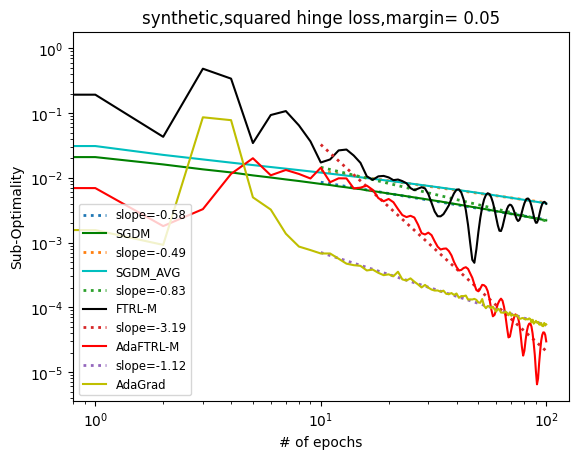}
\includegraphics[width=0.32\textwidth]{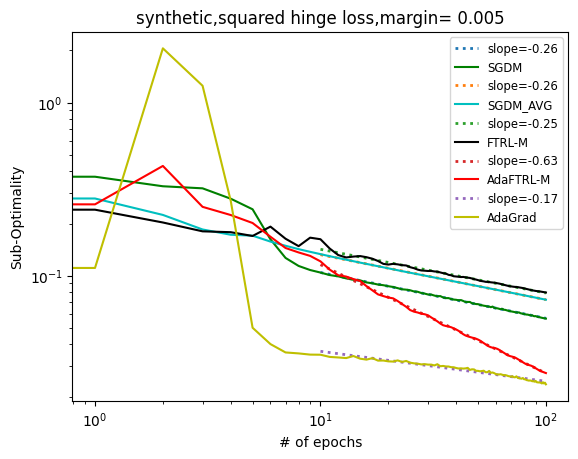}
\caption{Squared hinge loss for classification, objective value vs number of epoch.}
\label{fig:synth}
\end{figure*}
\begin{figure*}[!t]
\centering
\includegraphics[width=0.32\textwidth]{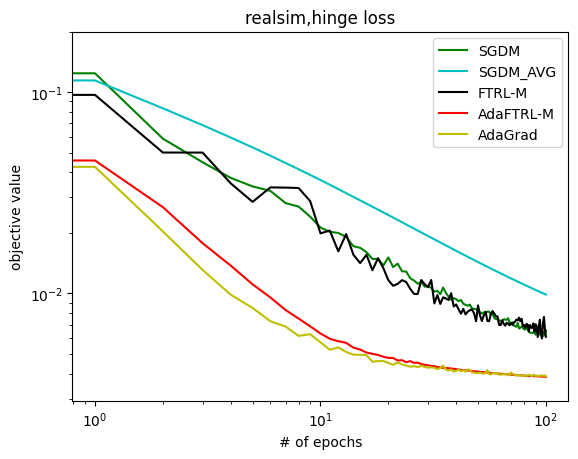}
\includegraphics[width=0.32\textwidth]{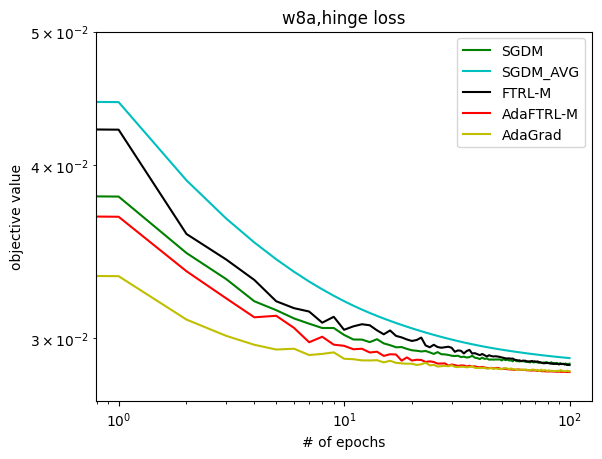}
\includegraphics[width=0.32\textwidth]{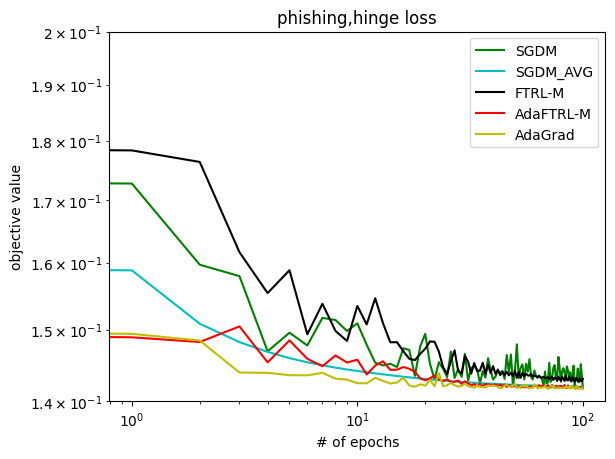}
\caption{Hinge loss for classification, objective value vs number of epoch.}
\label{fig:classification}
\end{figure*}

We have presented a family of FTRL-based SGDM algorithms, that exhibit optimal convergence of the last iteration. These algorithms are motivated by a new lower bound that shows that constant momentum SGDM is provably suboptimal to minimize convex Lipschitz functions. 
However, the theory guarantees only an improvement of $\ln T$, so it is unlikely to make a difference in practical applications. Yet, we also perform some experiments to show that FTRL-based momentum methods have also interesting empirical properties. 

We compare FTRL-M (Algorithm~\ref{alg:sgdm_0}, $\bgamma_t = \frac{c \cdot \mathbf{1}}{\sqrt{t}}$), AdaFTRL-M (Algorithm~\ref{alg:sgdm_0}, $\bgamma_t = \frac{\alpha}{\sqrt{\epsilon + \sum_{i=1}^{t} \alpha_i^2  \bg_i^2}}$)  with classic SGDM ($\beta = 0.9$), SGDM-AVG (averaged iterates of SGDM, $\beta = 0.9$), and AdaGrad~\citep{McMahanS10,DuchiHS10}. The initial stepsizes for all the algorithms were tuned with a fine grid-searching procedure.

\paragraph{Synthetic Data}
For the first experiment, we generate synthetic data and test the algorithms following the protocol in \citet{VaswaniBS19}. We generate a synthetic binary classification dataset with $n = 8000$ and the dimension $d = 100$. We make the data linearly separable with a margin, in which case the interpolation condition is satisfied. We train linear classifiers with the squared hinge loss: $f(\bw) = \sum_{i=1}^{n} \left(\max \left(0, 1-  \by_i \bw^T \bx_i\right)\right)^2$. Note that the loss function is smooth and $f(\bw^{\star}) = 0$. In this case, the optimal convergence rate is at least as fast as $1/T$.

We plot the suboptimality gap versus the number of epochs with different margin values in Figure~\ref{fig:synth}, in loglog plots. Also, we add a line to fit the curves, where the slopes represent the power of $t$. From Figure~\ref{fig:synth}, we observe that two adaptive algorithms AdaGrad and AdaFTRL-M bring faster convergence and AdaFTRL-M has the biggest slope in all the cases. Also, the performance of FTRL-M is on par with SGDM and SGDM-AVG.

\paragraph{Real Data}
We also test the algorithms on real datasets. We use classification datasets from the LIBSVM website~\citep{ChangL11}; \emph{real-sim}, \emph{w8a}, and \emph{phishing}.
The details of the datasets are in Appendix~\ref{sec:details_exp}.


We train linear classifiers with the hinge loss and no regularization: $f(\bw) = \sum_{i=1}^{n} (\max (0, 1-  \by_i \bw^T \bx_i))$. The stochastic gradients are obtained evaluating the subgradient on one example at the time. We repeat the experiments for 5 times for each algorithm and report the average of 5 repetitions.
We show the objective value versus the number of epochs in Figure~\ref{fig:classification}.

The results show that the algorithms with non-adaptive stepsizes tend to perform worse than the ones with adaptive stepsizes. 
Moreover, the performance of AdaFTRL-M is close to the last iterate of AdaGrad and sometimes outperforms all the other algorithms, especially in the last iterations.

\section{Conclusion}
\label{sec:conc}

We have presented an analysis of the convergence of the last iterate of SGDM in the convex setting. We prove for the first time through a lower bound the suboptimal convergence rate for the last iterate of SGDM with constant momentum after $T$ iterations. Moreover, we study a class of FTRL-based SGDM algorithms with increasing momentum and shrinking updates, of which the last iterate has optimal convergence rate without projections onto bounded domain nor knowledge of $T$. Furthermore, we present empirical results showing that FTRL-based SGDM with adaptive stepsize matches or outperforms the other similar algorithms in the last iterations.

In the future, we plan on studying the convergence in high probability of FTRL-based SGDM, similarly to the analysis in \citet{LiO20}.

\section*{Acknowledgements}
This material is based upon work supported by the National Science Foundation under the grants no. 1925930
	``Collaborative Research: TRIPODS Institute for Optimization and Learning'', 
	no. 2022446 ``Foundations of Data Science Institute'', and no. 2046096 ``CAREER:
	Parameter-free Optimization Algorithms for Machine Learning''.

\bibliography{../../../../learning}
\bibliographystyle{plainnat_nourl}
\appendix
\section{Lemma for the Proof of Theorem~\ref{thm:lower_bound}}
\label{sec:lemma_poly_bound}

\begin{lemma}
	\label{lemma:poly_bound}
	For any $1 \leq j \leq t \leq T$ and $0< \alpha \leq \frac{1}{2}$, we have $\frac{1}{T-j+1} \sum_{k=j+1}^{t} \frac{1}{j^{\alpha}} \leq  \frac{2}{T^{\alpha}}$. 
\end{lemma}
\begin{proof}
	First, we observe that 
	\begin{align*}
		\sum_{k=j+1}^{t} \frac{1}{k^{\alpha}} 
		&\leq \int_{j}^{t}  \frac{1}{x^{\alpha}} dx 
		= \frac{t^{1-\alpha} - j^{1-\alpha}}{1-\alpha} 
		= \frac{1}{1-\alpha}\frac{t^{2-2\alpha} - j^{2-2\alpha}}{t^{1-\alpha} + j^{1-\alpha}} \\
		&\leq \frac{1}{1-\alpha}\frac{(2-2\alpha)t^{1-2\alpha}  (t-j)}{t^{1-\alpha} + j^{1-\alpha}} 
		\leq \frac{2(t-j)}{t^{\alpha}},
	\end{align*}
	where in the second inequality we used the convexity of $f(x) = x^{2-2\alpha}, 0 < \alpha \leq \frac{1}{2}$. 
	
	Then, we claim $\frac{1}{T-j+1}\frac{t-j}{t^{\alpha}} \leq \frac{1}{T^{\alpha}}$. 
	
	Let $g(x) = \frac{x- j}{x^{\alpha}}$. The derivative 
	$g'(x) = \frac{1-\alpha + \frac{j}{\alpha x}}{x^{\alpha}}$ is positive for all $x> 0 $ and $j \geq 0$. So 
	it satisfies that $\frac{t-j }{t^{\alpha}} \leq \frac{T-j}{T^{\alpha}}$, which implies the claim. 
\end{proof}

\section{Lemma for the Proof of Theorem~\ref{thm:sgdm_0}}
\label{sec:lemma_ftrl}

\begin{algorithm}[h]
\caption{Follow-the-Regularized-Leader on Linearized Losses}
\label{alg:ftrl}
\begin{algorithmic}[1]
\STATE \textbf{Input:} Regularizers $\psi_1, \dots, \psi_T: \R^d \to(-\infty, \infty ]$. 
\FOR{$t = 1, \dots, T$}
\STATE $\bw_t \in \argmin_{\bw \in \R^d} \ \psi_t (\bw) + \sum_{i=1}^{t-1} \langle \bg_i , \bw \rangle$
\STATE Receive $\ell_t:  \R^d \to(-\infty, \infty ]$ and pay $\ell_t (\bw_t)$
\STATE Set $\bg_t \in \partial \ell_t (\bw_t)$
\ENDFOR
\end{algorithmic}
\end{algorithm}
The following lemma is a well-known result for FTRL~\citep[see, e.g.,][]{Orabona19}.
\begin{lemma}
\label{lemma:ftrl}
Let $\ell_t$ a sequence of convex loss functions. Set the sequence of regularizers as $\psi_t({\bx})= \left\| \frac{\bx_1 - \bu}{\sqrt{\bgamma_{t-1}}} \right\|^2$, where $\bgamma_{t+1} \leq \bgamma_t, \ t=1, \dots, T$. Then, FTRL (Algorithm~\ref{alg:ftrl}) guarantees
\begin{align*}
\sum_{t=1}^{T} \ell_t (\bx_t ) - \ell_t(\bu)
\leq  \left\| \frac{\bu - \bx_1 }{ \sqrt{\bgamma_{T-1}}} \right \|^2 + \frac{1}{2} \sum_{t=1}^{T}  \langle \bgamma_{t-1},  \bg_t^2\rangle~.
\end{align*}
\end{lemma}

\section{Proofs of Corollaries 2-4}
\label{sec:proofs_cor}

First, we state some technical lemmas.

\begin{lemma}{\citep[Lemma~4]{LiO19}}
\label{lemma:smooth}
Let $f: \R^d \to  \R$ be M -smooth and bounded from below, then for all $\bx \in \R^d$
\[
\| \nabla f(\bx) \|^2 \leq 2M (f(\bx) - \inf_{\by \in \R} f(\by))~. 
\]
\end{lemma}

%

\begin{lemma}{\citep[Lemma 14]{GaillardSV14}}
\label{lemma:sum_intergral_bounds_extra_term}
Let $a_0 > 0$ and $a_1, \dots, a_m \in [0,A]$ be real numbers and let $f:(0,+\infty)\rightarrow [0, +\infty)$ nonincreasing function. Then 
\[
\sum_{i=1}^{m} a_i f(a_0 + \dots  + a_{i-1})
\leq  \int_{a_0}^{\sum_{i=0}^m a_i} f(u) du + Af(a_0) ~.
\]
\end{lemma}
\begin{proof}
Denote by $s_t=\sum_{i=0}^{t} a_i$.
\begin{align*}
\sum_{i=1}^{m} a_i f(s_{i-1}) 
& =  \sum_{i=1}^{m} a_i f(s_i)  + \sum_{i=1}^{m} a_i (f(s_{i-1}) - f(s_i))\\
& \leq  \sum_{i=1}^{m} a_i f(s_i)  + A\sum_{i=1}^{m} (f(s_{i-1}) - f(s_i))\\
& \leq  \sum_{i=1}^{m} \int_{s_{i-1}}^{s_i} f(x) d x + A\sum_{i=1}^{m} (f(s_{i-1}) - f(s_i))\\
& \leq \int_{a_0}^{\sum_{i=0}^m a_i} f(u) du + Af(a_0) ~,
\end{align*}
where the first inequality holds because $f(x_{i-1}) \geq f(s_i)$ and $a_i \leq A$, while the second inequality uses the fact that $f$ is nonincreasing together with $s_i - s_{i-1} = a_i$. 
\end{proof}

We can now present the proofs of the Corollaries 2-4.

\begin{proof}[Proof of Corollary~\ref{cor:adaptive_norm} and Corollary~\ref{cor:adaptive_coord}]
By Lemma~\ref{lemma:sum_intergral_bounds_extra_term}, for adaptive stepsize $\bgamma_t =  \frac{\alpha \cdot \mathbf{1}}{\sqrt{\epsilon  + \sum_{i=1}^{t} \alpha_i^2 \| \bg_i \|^2}}$, we have 
\begin{align*}
\sum_{t=1}^{T} \bgamma_{t-1} \| \bg_t \|^2
= \sum_{t=1}^{T} \frac{\alpha\| \bg_t \|^2}{\sqrt{\epsilon  + \sum_{i=1}^{t-1} \| \bg_i \|^2}} 
\leq 2\alpha \sqrt{\sum_{t=1}^{T} \| \bg_t \|^2 } 
+  \frac{\alpha G^2}{\sqrt{\epsilon}}~. 
\end{align*}
Similarly for $ \bgamma_t = \frac{\alpha }{\sqrt{\epsilon  + \sum_{i=1}^{t} \alpha_i^2  \bg_i^2}}$, we have
\begin{align*}
\sum_{t=1}^{T} \langle \bgamma_{t-1},  \bg_t^2 \rangle 
= \sum_{j=1}^{d}\sum_{t=1}^{T} \frac{ \alpha \bg_{t,j}^2}{\sqrt{\epsilon  + \sum_{i=1}^{t-1} \bg_{i,j}^2}} 
\leq 2\alpha \sum_{j=1}^{d} \sqrt{\sum_{t=1}^{T} \bg_{t,j}^2}
+ \frac{\alpha dG_{\infty}^2}{\sqrt{\epsilon}}~. 
\end{align*}
\end{proof}

\begin{proof}[Proof of Corollary~\ref{cor:ada_smooth}]
By Corollary~\ref{cor:adaptive_norm}, we have 
\begin{align}
& \E\left[f(\bx_T) \right] - f^{\star}
\leq \frac{1}{T} \left[\left(\frac{\|\bx_1 - \bx^{\star} \|^2}{\alpha} + 2\alpha\right) \sqrt{\epsilon +  \E \sum_{t=1}^{T} \| \bg_t \|^2 } 
+ \frac{\alpha G^2}{\sqrt{\epsilon}}\right]~.\label{eq:adaptive_bound}
\end{align}
From the unbiasedness of the gradients, we have
\begin{align*}
&  \E \sum_{t=1}^{T} \| \bg_t \|^2 
\leq \E \sum_{t=1}^{T} \| \nabla f(\bx_t) \|^2 + T \sigma^2,
\end{align*}
and 
\begin{align*}
\E \sum_{t=1}^{T} \| \nabla f(\bx_t) \|^2
& \leq 2L \sum_{t=1}^{T} \E \left[f(\bx_t ) \right] - f^{\star}\\
& \leq 2L \left(\frac{\|\bx_1 - \bx^{\star} \|^2}{\alpha} + 2\alpha\right) \sum_{t=1}^{T}\frac{ \sqrt{ \epsilon + \E \sum_{i=1}^{t} \| \bg_i\|^2 }}{t} \\
& \leq 2L \left(\frac{\|\bx_1 - \bx^{\star} \|^2}{\alpha} + 2\alpha\right)\cdot \left(\sqrt{ \E \sum_{t=1}^{T} \| \bg_t\|^2+ \epsilon} + \frac{\alpha G^2}{\sqrt{\epsilon}}\right) \ln T,
\end{align*}
where in the second inequality we used Lemma~\ref{lemma:smooth} and Holder's and Jensen's inequalities in the third inequality.
	
Solve for $\E \sum_{t=1}^{T} \| \bg_t \|^2$ to have
\begin{align*}
& \E \sum_{t=1}^{T} \| \bg_t \|^2 \\
& \leq 4L^2\left(\frac{\|\bx_1 - \bx^{\star} \|^2}{\alpha} + 2\alpha\right)^2 \ln^2 T 
+ 4L\sqrt{\epsilon} \left(\frac{\|\bx_1 - \bx^{\star} \|^2}{\alpha} + 2\alpha\right) \ln T + \frac{2\alpha G^2}{\sqrt{\epsilon}} + 2T \sigma^2~. 
\end{align*}
	
Putting it back to \eqref{eq:adaptive_bound}, we get the stated bound. 
\end{proof}

\section{Details of Experiments}
\label{sec:details_exp}

\begin{table}[h]
\centering
\caption{Real datasets}
\begin{tabular}{ ||c | c | c||  }
\hline
Name & \# of Samples& \# of Features \\ [0.5ex]
\hline
real-sim & 72,309& 20,958 \\
w8a& 49,749 & 300 \\
phishing & 11,055 & 68 \\
\hline
\end{tabular}
\label{table:data}
\end{table}

\end{document}